\documentclass[twoside,11pt]{article}

\usepackage[abbrvbib, preprint, nohyperref]{jmlr2e}

\usepackage[
	colorlinks,
	pdfpagelabels,
	pdfstartview = FitH,
	bookmarksopen = true,
	bookmarksnumbered = true,
	linkcolor = blue,
	plainpages = false,
	hypertexnames = false,
	citecolor = red] {hyperref}
\usepackage[toc,page]{appendix}
\usepackage[utf8]{inputenc}
\usepackage{xcolor}
\usepackage{MnSymbol}
\usepackage{enumitem}
\usepackage[english]{babel}
\usepackage{amsmath}
\usepackage{bbm}

\usepackage[capitalize]{cleveref}

\crefname{enumi}{}{}
\crefname{equation}{}{}

\def \R {\mathbb {R}}

\newcommand{\eps}{\varepsilon}
\DeclareMathOperator*{\esssup}{ess\,sup}

\ShortHeadings{Deep limits of neural networks}{Avelin and Karlsson}
\firstpageno{1}

\begin{document}

\title{Deep limits and cut-off phenomena for neural networks}

\author{\name Benny Avelin \email benny.avelin@math.uu.se \\
       \addr Department of Mathematics\\
       University of Uppsala\\
       Box 256, 751 05 Uppsala, Sweden
       \AND
       \name Anders Karlsson \email anders.karlsson@unige.ch \\
       \addr Department of Mathematics\\
       University of Uppsala\\
       Box 256, 751 05 Uppsala, Sweden \\
       \addr Section de mathématiques\\
       Université de Genève\\
       Case Postale 64, 1211 Genève 4, Switzerland}

\editor{}

\maketitle

\begin{abstract}%
We consider dynamical and geometrical aspects of deep learning. For many standard choices of layer maps we display semi-invariant metrics which quantify
differences between data or decision functions. This allows us, when considering random layer maps and using non-commutative ergodic theorems, to deduce that certain limits exist when letting the number of layers tend to infinity. We also examine the random initialization of standard networks where we observe a surprising cut-off phenomenon in terms of the number of layers,
the depth of the network. This could be a relevant parameter when choosing an appropriate number of layers for a given learning task, or for
selecting a good initialization procedure. More generally, we hope that the notions and results in this paper can provide a framework,
in particular a geometric one, for a part of the theoretical understanding of deep neural networks.
\end{abstract}


\section{Introduction}

In this paper, we develop a geometric toolkit which we propose as a means to study neural networks, in particular as the depth tends to infinity. Viewed in this sense, we can deduce properties of the limit of neural networks as the number of layers go to infinity provided that the layers preserve certain distance functions (metrics).

As a starting point we will consider neural networks with random layers, which can occur from dropout \cite{Dropout}, Bayesian neural networks \cite{Bayesian}, neural networks with noise (Neural SDE) \cite{LXSCKH}, or simply random initialization. The assumptions on the dependence between subsequent layers is weak and we only assume stationarity.

Our analysis shows that under the assumption of stationarity, if one can find a metric space for which the ``layer transformations'' of the neural networks is non-expansive, then the limit and its growth rate can be described using powerful tools from ergodic theory.

\subsection{Background}
As is by now well known, certain deep networks perform better than their shallow counterpart, see for instance \cite{HZRS16}. Fairly recently \cite{DoubleDip0} observed a phenomenon later dubbed ``Deep Double Descent'' in \cite{DoubleDip}. The deep double descent means that for deep networks, after a certain width threshold the generalization properties becomes better and better even though the class of networks becomes increasingly complex, \cite{Hornik}. However, as mentioned above, also deeper networks seems better in terms of generalization, which suggests there is a regularizing effect of depth under certain conditions, similar to deep double descent.

The wide limit of neural networks is fairly well studied, see for instance \cite{NealPrior, Mei, NTK, RVE}. However the deep limit is not a particularly well defined concept and there are many different ways to view it. One of the more practically successful ones are the Neural ODEs, introduced by \cite{CRBD18}, which can be seen as a deep limit of residual networks (\cite{TG18,AN19}).
The discrete model for these continuous neural networks can be formulated as
\begin{align} \label{eq:small-time-step}
  x_{t_{i+1}} = x_{t_i} + \frac{1}{n}T_i(x_{t_i}), \quad i = 1,\ldots,n
\end{align}
where $T_i$ represents a layer in the neural network. In the case of Neural SDEs (\cite{LXSCKH, TR}), or in the Bayesian framework (see for instance \cite{Bayesian}) we can view each layer as being random. Furthermore, the special case of i.i.d. random layers is present in the random initialization of the network, and is in fact a very important aspect to understand when it comes to training neural networks, \cite{Sutskever}.

A key observation in the above formulation is that the discrete form \cref{eq:small-time-step} represents an approximation of the ODE
\begin{align}
  \frac{\partial x(t)}{\partial t} = T_t(x_t), \quad t \in [0,1],
\end{align}
i.e. the discrete system is an approximation of a fixed time horizon ODE, with a time-step of size $1/n$. Another point of view is to consider a fixed time-step and consider the behavior of the system as $n \to \infty$, i.e. of
\begin{align} \label{eq:fixed-time-step}
  x_{t+1} = x_{t} + T_t(x_{t}), \quad t = 1,\ldots,n.
\end{align}

In a sense the networks in \cref{eq:fixed-time-step} represent a duality of thought compared to \cref{eq:small-time-step}, in that either we consider a deep neural network as consisting of many distinct layers or we consider a single (time dependent) network being repeated in a recurrent fashion, ``layer'' is represented by time. This connects deep neural networks to the concept of a recurrent neural network and in a sense they are the same, specifically neural ODEs and neural SDEs which are even trained using recurrent back-propagation (real-time recurrent learning), see \cite{CRBD18,LXSCKH, RF,R,TR, WZ}.

In the context of Bayesian neural networks, there has recently been some progress in establishing the deep limits of these as certain Gaussian processes, see \cite{APH,Dunlop,Duvenaud}.

According to \cite{EMW}, at the continuous level many machine learning models are the gradient flow of a reasonably nice functional, and they argue that this is a reason for models such as \cref{eq:small-time-step} (ResNets) are numerically stable, \cite{H,HBFS}. They also suggest that for \cref{eq:fixed-time-step} one should expect trouble since there is no continuum limit. This is true to some degree, but one should keep in mind that even a standard ResNet does not have the scaling factor $1/n$ in front of $T_n$, thus one could argue that it is more reasonable to consider the limit of fixed time-step dynamics.
Of course this may not have a limit but perhaps a certain rescaling does, for instance, one could consider
\begin{align*}
  \lim_{n\to \infty} \frac{1}{n} \log(x_{n}),
\end{align*}
or
\begin{align*}
  \lim_{n \to \infty} \frac{1}{n} x_n.
\end{align*}

\subsection{Our contribution}
In this paper we take the viewpoint of \cref{eq:fixed-time-step} and we rephrase the update equation as $x_{n+1}=T_n(x_n)$. The problem is now one of discrete (possibly chaotic) dynamical systems.
The main contribution of our paper is that we develop a framework to study deep neural networks from a geometric perspective. Specifically it allows us to read out certain stability properties whenever there exists a metric which is preserved by the network layers. Depending on the metrics involved, it can tell us if the networks tend to satisfy some regularity as we go deeper, even though in principle the networks can become arbitrarily complex mappings (\cite{DeepUniversal}). This serves as an indication as to when one would observe the ``Deep Double Descent'' phenomenon with respect to depth.

Note that we study the network directly without worrying about how we obtained said network. We consider stationary sequences of layers, which can cover for instance a sequence correlated layers for which the correlation tends to zero as we go deeper.

In the context of independent, identically distributed (i.i.d.)~random layers, which corresponds to the random initialization of the weights in the network, we perform a few experiments where we observe a cut-off phenomenon.
This means that for a given type
of neural network there is a certain number of layers where the network behaves very differently. We call this its \emph{cut-off depth}. The cut-off phenomenon was first discussed in
card shuffling by \cite{AD} who showed that seven shuffles are enough. It remains to understand the full significance of
the cut-off depth for deep learning.
\begin{figure}
    \centering
    \includegraphics[width=5cm]{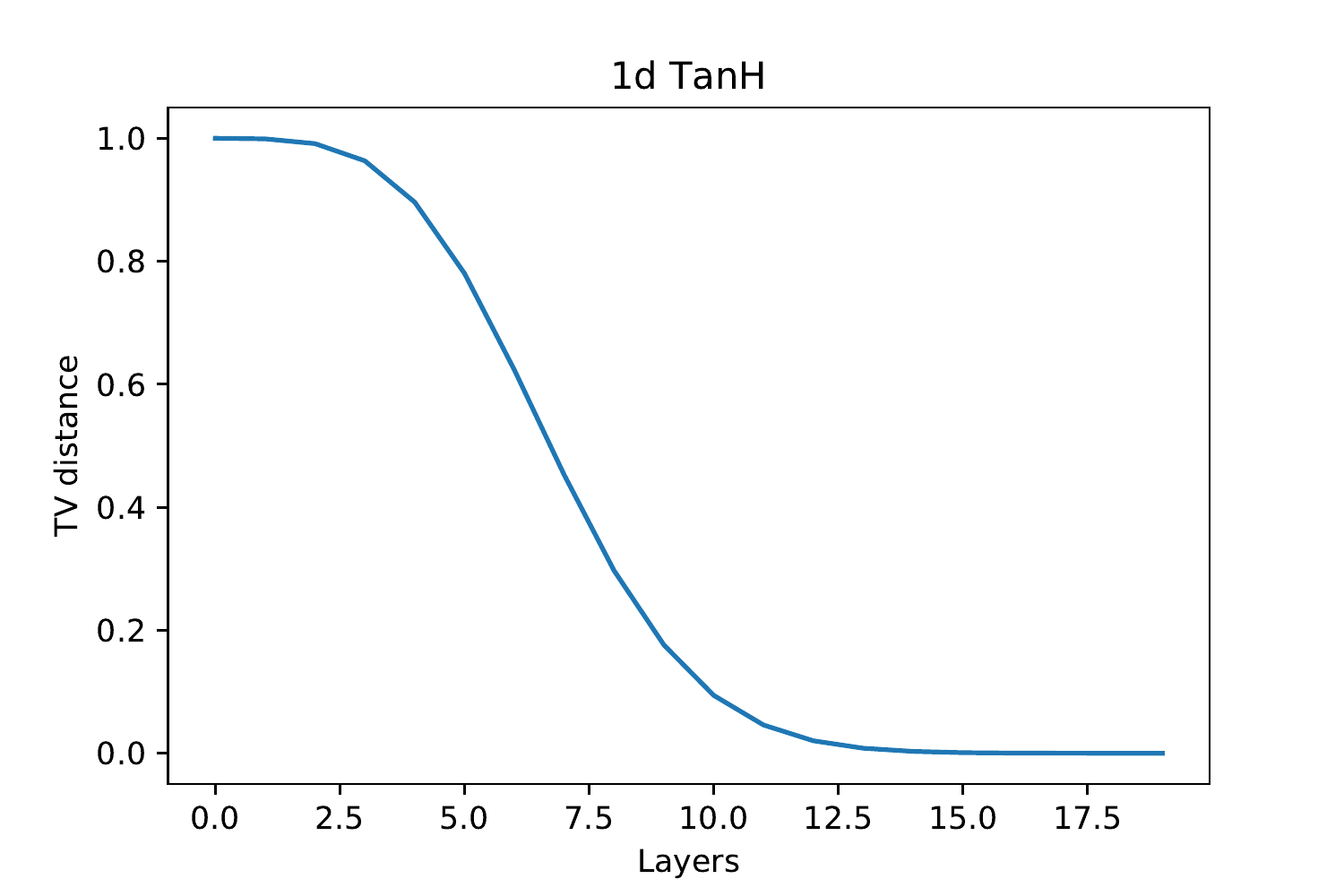}
    \includegraphics[width=5cm]{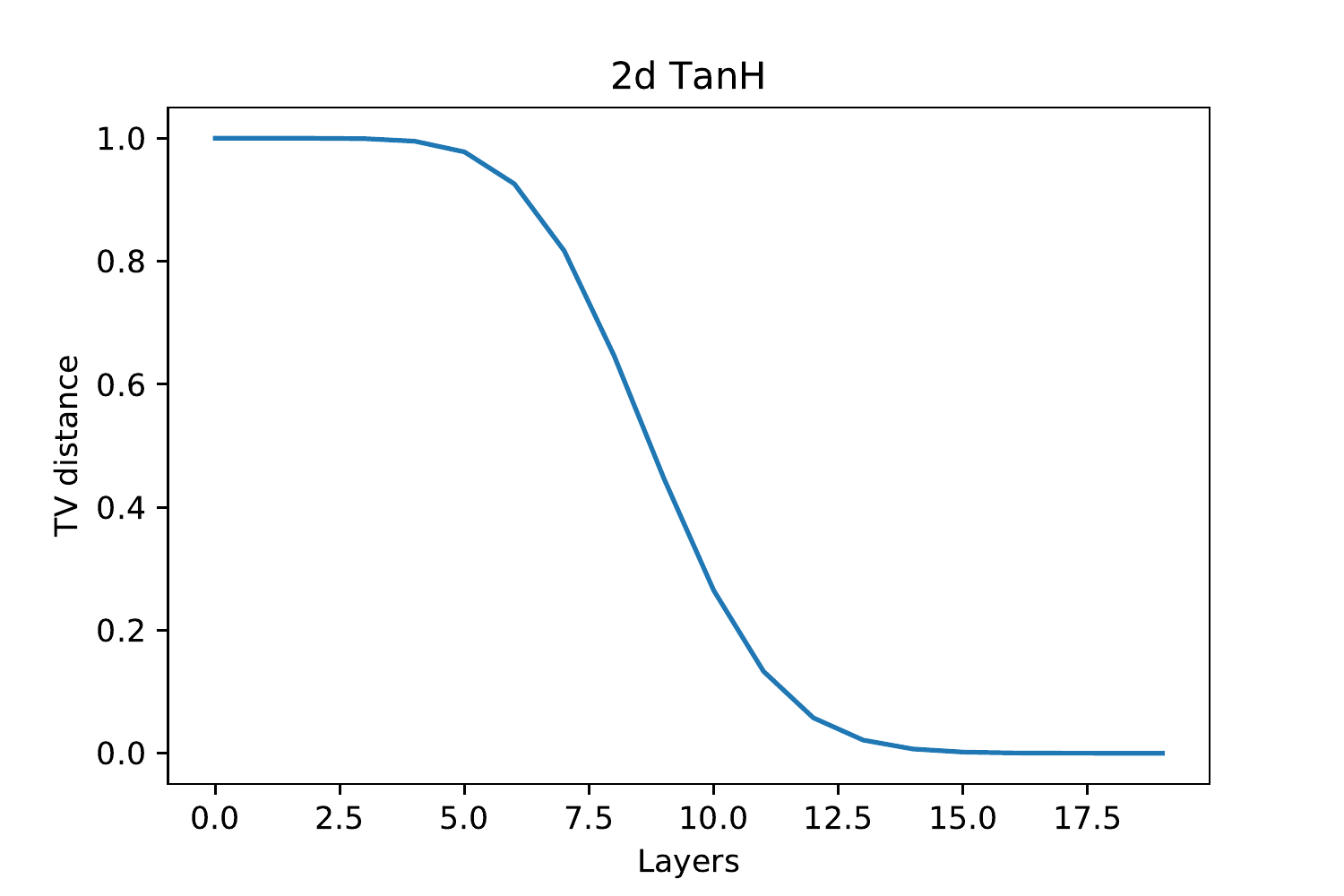}
    \caption{The cut-off phenomenon for neural network mixing with TanH activation. The figure describes the total variation (TV) distance to the equilibrium measure as a function of the number of layers. The leftmost figure is width 1 and the rightmost figure is width 2.}
    \label{fig:threshold:intro}
\end{figure}

In summary, part of the rationale for our study is:
\begin{itemize}
  \item understanding how the function class evolves as the number of layers go to infinity.
  \item understanding if the discrete time-step model leads to a cut-off phenomenon? That is, is there a threshold amount of layers for which the space of output functions of the network increase significantly?
  \item the randomness of our mappings $T_i$ is a way to model a host neural networks, including standard regularization techniques, Bayesian networks, other noise injected models or random initialization.
\end{itemize}

\section{The dynamics of deep neural networks}

Let $X\subseteq\mathbb{R}^{d}$ denote the space which contains the input information
as well as the intermediate data traveling through the
hidden layers, in the notation below, $x_{n}\in X$ for all $n$.
It could be the full vector space, or a subset such as the positive
cone or a unit cube.
Each layer defines a transformation $T:X\rightarrow X$
typically of the form:
\[
    T:x\mapsto\sigma(Wx+b),
\]
where $W$ is a matrix, called \emph{weights}, and $b$ a vector,
called \emph{bias vector}. Note that this is just an example and we can in fact have $T$ being a small network. The \emph{activation function} $\sigma$
is a non-linear function $\sigma:\mathbb{R}\rightarrow\mathbb{R}$
that is fixed for the whole network and is applied to each coordinate.
Two standard choices for $\sigma (t)$ are $\max\{0,t\},$ called the rectified linear
unit (ReLU), and $\tanh(t)$ (TanH). The former has been observed
to often work very well in practice and the latter has the advantage
of being a diffeomorphism $\mathbb{R}\rightarrow(-1,1)$. Other common
 choices are $\min\{1,\max\{0,t\}\}$ and the sigmoid or logistic function $1/(1+e^{-t})$.

As mentioned in the introduction, in deep learning one uses several
layers, sometimes even up to a thousand. We denote these $n$ layer transformations,
$T_{1},...,T_{n}$. In order to gain some theoretical understanding for the role
of the number of layers, the depth, in neural network, we are interested
in what happens to the neural network when $n$ is large, or $n\rightarrow\infty$.
There are now in fact two possible dynamics to look at, first, new
layers are added at the end just before the final output, or second, new
layers are added just after the initial input. For a given initial
input $x_{0}\in X$, these two dynamics correspond to, respectively,
\begin{align} \label{eq:appending:layers}
    x_{n}=T_{n}T_{n-1}...T_{1}x_{0}
\end{align}
and
\begin{align} \label{eq:inserting:layers}
    x_{n}=T_{1}T_{2}...T_{n}x_{0}.
\end{align}
We can view these dynamics as a representation of transfer learning. Where \cref{eq:appending:layers} corresponds to adding a new layer at the end, which is the standard way transfer learning is used. On the other hand \cref{eq:inserting:layers} corresponds to keeping the last layers and inserting a new layer at the beginning, this can be seen as transfer learning in the context of domain adaptation.
Note that if we take the maps randomly, or more precisely independent
and identically distributed (i.i.d.), then for each fixed $n$ and given
$x_{0}$ the distributions of $x_{n}$ are the same. But if we
study the dynamics, the evolution $x_{n}$ of an individual $x_{0}$
then these two dynamics behave differently because of the non-commutativity
of the layer transformations.

In addition to these two ``input dynamics'' we will also consider
``output dynamics''. At the last hidden layer one often has a \emph{decision
function} $f:X\rightarrow Y$, where $Y$ is also a subset of some
vector space. For example, $f$ is in many cases an indicator function of some set or a smooth approximation thereof.

Like the transpose of matrices, or more precisely the adjoint of operators which
transforms linear functionals instead of vectors, we can look at
the effect of the layers on the output function. We can thus let the layers
and dynamics transform the decision function to a function that is then directly applied
to the initial input. In other words, we are pulling back the decision function to the input as it were.
This is the dual dynamics. In formulas,
\[
(T^{*}f)(x):=f(Tx).
\]
Note that orders get reversed just like for the transpose of matrices:
\[
(T_{1}^{*}T_{2}^{*}...T_{n}^{*}f)(x_{0})=f(T_{n}T_{n-1}...T_{1}x_{0})
\]
and corresponding to the second dynamics above:
\[
(T_{n}^{*}T_{n-1}^{*}...T_{1}^{*}f)(x_{0})=f(T_{1}T_{2}...T_{n}x_{0}).
\]

\section{Framework and strategy}

We introduce a geometric viewpoint on neural networks. In several of the most poplar
network models in deep learning, we exhibit an associated metric space on which the
layer maps act as non-expansive maps.

Once we have this one could potentially use the contraction mapping principle, in
the version of a sequence of maps, which composed has a summable contraction constant.
We refer to the review by \cite{DF99} for more information.

But often in our context such strong contraction property is not available, but then
we instead have the non-commutative ergodic theorem
in \cite{GK20}, recalled below as \cref{multergthm}, as a main tool.

In order to probe the neural network to understand a bit better the role of the number of layers, we apply these ergodic theorems to stationary sequences of layer maps. In
experiments we observe a cut-off phenomenon, see \cref{sec:cutoff}.

We will in this paper analyze the following two dynamical systems:
\[
x_{n}=T_{1}T_{2}...T_{n}x_{0}
\]
or
\[
f_{n}:=T_{1}^{*}T_{2}^{*}...T_{n}^{*}f.
\]
These cover the two dynamics above: adding layers at the beginning
or at the end, respectively. It is this order that corresponds to random
walks, where each step $y_{n}$ is not far from $y_{n-1}$. Expressed
differently, these orders will make it possible to  extract some coherent
asymptotic behavior of $x_{n}$ in the first case and $f_{n}$ in
the second. (For certain quantities such as the probability distribution
or the basic growth the order does not matter. This will be seen later).

\section{Ergodic theorems}

To better understand where the tools of ergodic theory come from let us recall the classical law of large numbers (LLN). It asserts that for i.i.d. random variables $X_i$ with $E[|X_i|] < \infty$,
\begin{align*}
	\frac{1}{n}(X_1 + X_2 + \ldots + X_n) \to E[X_1] \quad \text{almost surely.}
\end{align*}

One can wonder if there is a similar law when the random variables are not commuting, for example the product of randomly selected matrices. Notice that in the non-commutative case it is not
obvious how to form an average, and one complication is clear \emph{a priori}: the limits typically will depend on the order of the maps, in contrast to the classical LLN. The results concerning such non-commutative ergodic theorems are of two types, subadditive ergodic theorems and multiplicative ergodic theorems. Let us begin by describing Kingman's subadditive ergodic theorem \cite{Ki} as it is a generalization of the LLN to operations that are subadditive. As a simple special case (the Furstenberg-Kesten theorem) let us consider a sequence of i.i.d. random matrices $A_i \in \R^{N \times N}$ and the functions $a(i,j) =  \log ||A_i \ldots A_{i+j} ||$. This is subadditive in $i$, i.e. $a(1,i+j) \leq a(1,i) + a(i,j)$. The subadditivity follows from the basic norm inequality. In this case, Kingman's subadditive ergodic theorem asserts that
\begin{align*}
	\lim_{n \to \infty} \frac{a(1,n)}{n} \to \ell.
\end{align*}
The limiting value is deterministic like in the LLN but on the other hand there is no good formula for its value.

Multiplicative ergodic theorems are stronger in the sense that they say more about the limit, loosely speaking they give an asymptotic direction of the limit.
The prototypical multiplicative ergodic theorem is Oseledets' theorem \cite{Os}, which relates to random products of matrices.

The setting for these ergodic theorems and for us here is that of integrable ergodic cocycles. This corresponds to
stationary sequences in probability theory language, a special case of which is the i.i.d
with finite first moment setting. Another special case
is the mixing case where one has asymptotic independence.

We will call such
integrable ergodic cocycles simply \emph{stationary sequences}
(with the integrability condition implicitly understood).
Technically it means that we have an underlying probability measure space $(M,\mu)$, $\mu(M)=1$, and
a measurable transformation $L:M\rightarrow M$ that preserves the measure $\mu(L^{-1}A)=\mu(A)$ and is ergodic, which means that every $L$-invariant set has measure
either $0$ or $1$.
Finally we have a measurable map from $M$ to a set of layer maps, $m\mapsto T_m$, so that for the
metric under consideration, all distances involved are measurable and
\begin{align*}
    \int_M d(T_m x,x)d\mu(m) <\infty.
\end{align*}
Note that this condition is independent of
$x$ since each map $T_m$ is non-expansive in the metric.

Consider now a stationary sequence of random matrices $T_i \in \R^{N \times N}$, then Oseledets' theorem states that
\begin{align*}
	\lim_{n \to \infty} \frac{1}{n} \log \|T_1 \ldots T_n x\| = \lambda
\end{align*}
the value of $\lambda$ is not random but depends on $x$, and there is at most $N$ different values, which are called Lyapunov exponents. A good way of understanding the Lyapunov exponents is to consider the special case where all the $T_i$'s are the same matrix, in this case we just taking powers of this matrix and the Lyapunov exponents are just the logarithm of the absolute value of the eigenvalues.

Oseledets' theorem only applies to linear maps (or the derivative cocycle of a diffeomorphsism), but we need to analyze non-linear maps since we have an activation function.
To describe and understand the dynamics of such more general settings we should define something
quantitative, like a norm or a metric. For example, to measure how close
one decision function is to another, or the distance between two information vectors.
These norms and distances should be preserved or at least not increase
when a transformation is applied to any two points. Specifically, Let $d$
denote a metric on either $X$ or some space $L$ of functions $X\rightarrow Y$.
We are interested in metrics that are \emph{semi-invariant}, i.e. metrics that for a given map $U$ satisfies
\begin{align*}
  d(U(z), U(w)) \leq d(z,w).
\end{align*}
for all $z$ and $w$ in $X$ or $L$, respectively. Correspondingly we will call a mapping that has a semi-invariant metric, as a \emph{non-expansive} map with respect to said metric.

We choose layer transformations $T_{k}$ at random (a
stationary sequence).
We let $z_{n}$ denote either of the two random processes,
$$T_{1}T_{2}...T_{n}x_{0}
 \,\,\textrm{or} \,\, T_{1}^{*}T_{2}^{*}...T_{n}^{*}f.$$
Fix\emph{ $d$} a semi-invariant
metric wiht respect to all the layer maps $T_k$. Like in the matrix example above, it is easy to see that $d(z_{0},z_{n})$ is then a subadditive
process, see for example \cite{KL11,GK20} (this is most clearly written
with ergodic theoretic formalism).
Then by Kingman's subadditive ergodic theorem \cite{Ki} we have that
\[
\lim_{n\rightarrow\infty}\frac{1}{n}d(z_{0},z_{n})
\]
exists a.s. This can be viewed as the existence of a basic regularity or growth.
This holds for all the dynamics considered, including the reverse orders.

Multiplicative ergodic theorems refine this convergence, in the way that it predicts
a directional behavior of $z_{n}$ (compare again with the matrix case above). The
precise statement, which in fact generalizes Oseledets' theorem above, is as follows:

\begin{theorem} \label{multergthm} (\cite{GK20})
    For any stationary sequence of maps
    as above, with $z_n$ denoting the orbit, there exists a.s.
    a metric functional $h$ (that is \emph{a priori} random) such that
    \[
        \lim_{n\rightarrow\infty}-\frac{1}{n}h(z_{n})=\lim_{n\rightarrow\infty}\frac{1}{n}d(z_{n},z_{0}).
    \]
\end{theorem}

In order to apply this general theorem we need to understand what the metric functionals are for a given metric, see \cref{app:metric:functionals}.

In summary, the strategy we suggest is as follows:
\begin{enumerate}
    \item Given a selected type of layer maps, find a metric space on which the maps act non-expansively.
    \item Determine the metric functionals of this space
    \item Apply the non-commutative ergodic theorem and interpret the result in the given situation.
\end{enumerate}

Moreover, we believe that already the first step, the metric setting, will have other interests
for deep learning, different from the application of ergodic theorems.

\section{Metric spaces} \label{sec:metric}

%%%%%%%%%%%%%%%%%%%%%%%%

A metric space is a set with a distance function. In recent decades it has been realized that significant geometrical arguments work in such a general setting even
without any differentiability. The subject is now often called metric geometry and has begun to infiltrate other areas, such as computer science.

Here we now turn to describing some metrics that are relevant in our context. The most basic metric space is the euclidean space of some finite dimension. Here the metric is
given by $d(x,y)=||x-y||$ where the norm comes from a scalar product.

Convex cones in vector spaces admit several useful choices of metrics with
important maps being non-expansive, we refer to the
excellent book \cite{LN12} for full information. We give some special cases of this here.
Given a finite dimensional real vector space, consider the set $X$ of vectors having all its entries positive. Thus $X$ is a generalized first quadrant, a convex cone.
Consider the following expression
\[
	f(x,w)=\log\left(\max_{i}\frac{w(i)}{x(i)}\right)
\]
in terms of the coordinates $x(i)$ and $w(i)$ of the vectors $x,w\in X$. Note that $f$ is asymmetric in its arguments, so in order to build a metric, we need to symmetrize it. It is also clearly not necessarily positive. There are two options, the Thompson metric and the Hilbert projective metric (``projective'' refers to that it is a distance function between lines, while the distance between two proportional vectors are easily seen to be 0. The triangle inequality is not obvious, but we refer again to \cite{LN12} for that.)
The Thompson's metric is defined as:
\[
	d(x,w)=\log \left (\max\left\{ \max_{i}\frac{x(i)}{w(i)},\max_{i}\frac{w(i)}{x(i)}\right\} \right )
\]
in terms of the coordinates $x(i)$ and $w(i)$ of the vectors $x,w\in X$. This makes the cone $X$ into a metric space and its main feature is that any order-preserving, subhomogeneous map of the cone into itself is a non-expansive map in this metric, for more information see \cref{sec:positive}.

The Hilbert metric is instead
\[
	d(x,w)=\log \left( \max_{i}\frac{x(i)}{w(i)} \cdot \max_{i}\frac{w(i)}{x(i)} \right).
\]
This is also obviously a symmetric expression. On the other hand it clearly is $0$ on rays, $d(x,\lambda x)=0$, $\lambda > 0$. More generally if we define the equivalence relation $x \underset{X}{\equiv} y$ if there is a $\lambda > 0$ such that $\lambda x = y$, then $((X,\underset{X}{\equiv}),d)$ is a metric space. Furthermore, order-preserving, homogeneous maps are non-expansive in this metric.

In one possible approach to Oseledets' theorem, one looks at the
associated action of the matrices on the space of positive scalar products on the
underlying vector space, see \cite{KL11} and references therein.
Since our maps are not linear we cannot do the same. We will instead
suggest to look at a much larger space, namely distance functions
either on $\mathbb{R}^{N}$ or the cube. To illustrate these ideas
we fix the cube $X=[-1,1]^{N}$ and use TanH which map the whole vector
space diffeomorphically onto the open cube. Let $M$ be the set of
distance functions on $X$ bi-Lipschitz equivalent to the original
distance defined by a norm $\left\Vert \cdot\right\Vert $. Here is
a metric on $M$:
\[
    D(d_{1},d_{2})=
    \log\left(\max \left \{ \sup_{x\neq y}\frac{d_{2}(x,y)}{d_{1}(x,y)}, \sup_{x\neq y}\frac{d_{1}(x,y)}{d_{2}(x,y)} \right \}\right).
\]
Notice that for two distance functions that are $K$-bi-Lipschitz to
each other, $0<D(d_{1},d_{2})<\infty$. For the optimal $K>1$ their
distance would be $\log K$. This function is clearly symmetric and $D(d_{1},d_{2})=0$
if and only if $d_{1}=d_{2}.$ The triangle inequality is also satisfied
because of the obvious properties of sup and log, like for the Thompson
metric above. So $(M,D)$ is a metric space.

If T are diffeomorphisms then obviously this metric is invariant under
$T^{*}$ which maps $d$ to $T^{*}d(x,y):=d(Tx,Ty)$, since $T$ just permutes the underlying set, leaving the
supremum invariant. Suppose T is
merely injective (otherwise there would occur a division by zero), then $T$ is
non-expansive in view of the inequality:
\[
    \log \left (\sup_{x\neq y}\frac{d_{2}(Tx,Ty)}{d_{1}(Tx,Ty)} \right ) \leq\log \left (\sup_{x\neq y}\frac{d_{2}(x,y)}{d_{1}(x,y)} \right )
\]
since $TX \subset X$.

From this perspective, it seems that injectivity is important, TanH
and invertible matrices gives (non-surjective) isometries of $(M,D)$.
Other activation functions would also be possible, but not ReLU.

In one dimensional dynamics, say in the study of diffeomorphisms $f$ of
a finite interval $I$, one finds the following quantitative measures
of distortion and distance. First out is
\[
\text{Var}(\log(\text{D}\, f))=\int_I \, \left|\frac{\text{D}^{2}\,f}{\text{D}\,f}\right| \, dx
\]
This has subadditive properties under composition of maps, see \cite{Na}. In particular
if one takes a random composition and divides by n, this converges a.e.
to a deterministic value, called the \emph{asymptotic variation}.

Another measure of distortion is
\[
    D(f)=\sup_{x,y\in I}\left|\log\left|\frac{f'(x)}{f'(y)}\right|\right|,
\]
see \cite{DKN07}. $D$ is subadditive with respect to compositions $f^{n}$ and symmetric with respect to $f$ and the inverse $f^{-1}$. We can make this into a metric as follows, \emph{the distortion metric}:
\[
    d(f,g)=\sup_{x,y\in I}\left|\log\left|\frac{g'(x)}{f'(x)}\frac{f'(y)}{g'(y)}\right|\right|
\]
but one can also consider a Thompson version. Furthermore, \cref{multergthm} also holds for asymmetric distances, see \cite{GK20}, this allows us to even consider just half of it, i.e taking away $y$.

We can extend the distortion metric to higher dimension by considering the Jacobians instead of the derivatives. Let us consider diffeomorphisms $f,g: \Omega \to \Omega$ in $\R^N$, then we can consider the Jacobian matrix $J_f = \left \{\frac{\partial f_i}{\partial x_j} \right \}$ and the Jacobian determinant $|J_f|$. Let us define the pseudo-metric
\begin{align*}
    D(f,g) = \sup_{x,y \in \Omega}\left|\log\frac{|J_f(x)|}{|J_g(x)|}\frac{|J_g(y)|}{|J_f(y)|}\right|
\end{align*}
Diffeomorphisms $h:\Omega \to \Omega$ are isometries, i.e.
\begin{align*}
    d(f \circ h, g \circ h)
    &=
    \sup_{x,y \in \Omega}\left|\log\frac{|(J_f \circ h)(x) J_h(x)|}{|(J_g \circ h)(x) J_h(x)|}\frac{|(J_g \circ h)(y) J_h(y)|}{|(J_f \circ h)(y) J_h(y)|}\right|
    \\
    &=
    \sup_{x,y \in \Omega}\left|\log\frac{|(J_f \circ h)(x)||J_h(x)|}{|(J_g \circ h)(x)||J_h(x)|}\frac{|(J_g \circ h)(y)||J_h(y)|}{|(J_f \circ h)(y)||J_h(y)|}\right|
    \\
    &=
    \sup_{x,y \in \Omega}\left|\log\frac{|J_f(x)|}{|J_g(x)|}\frac{|J_g(y)|}{|J_f(y)|}\right| = d(f,g)
\end{align*}
The second to last follows from $\det(AB)=\det(A)\det(B)$ and the last step follows from the fact that $h$ is a diffeomorphism.

In the case where $h: \Omega \to \Omega$ is not a diffeomorphism with non-singular Jacobian, they are non-expansive in the view of
\begin{align*}
d(f \circ h, g \circ h)
&=
    \sup_{x,y \in \Omega}\left|\log\frac{|(J_f \circ h)(x)||J_h(x)|}{|(J_g \circ h)(x)||J_h(x)|}\frac{|(J_g \circ h)(y)||J_h(y)|}{|(J_f \circ h)(y)||J_h(y)|}\right|
    \\
    &\leq
    \sup_{x,y \in \Omega}\left|\log\frac{|J_f(x)|}{|J_g(x)|}\frac{|J_g(y)|}{|J_f(y)|}\right| = d(f,g).
\end{align*}

%%%%%%%%%%%%%%%%%%%%%%%%%%%%%%%%%%%%%%%%%%%%%
\section{Main results}

In this section we will use our previously outlined strategy to derive conclusions about the deep limit of neural networks. We will be employing both subadditive and multiplicative ergodic theorems to do so.

\subsection{Positive models} \label{sec:positive}

We begin our exposition into explicit examples, by considering what we call positive models. That is, layers that can only produce positive output. To be specific, let us take for $X$ the cone of vectors in $\R^N$ with all
coordinates $\geq 0$. The layer maps $T(x)=\sigma(Wx+b)$ are such
that $W$ is a matrix with every entry $\geq0$ and same for $b$.
Finally $\sigma$ is an activation function which is increasing and
satisfies $\sigma(\lambda x)\leq\lambda\sigma(x)$ for every $\lambda>0$ and $x \geq 0$, (note that this implies that $\sigma(0)=0$). For example, ReLu, TanH, and the sigmoid.

Note the following properties:
\begin{itemize}
    \item $W$ preserves the cones, i.e. maps $X$ into $X.$ So does $b$ and finally also $\sigma.$ Therefore $T:X\rightarrow X$. (With ReLU this is true without assumptions on $W$ and $b$.)
    \item In fact, more is true, if $x\leq y$ in the partial order defined by the cone (i.e. all components of $x$ are smaller or equal to those of $y)$ then by the positivity of $W$ and $b$ as well as the increasing property of $\sigma$, it holds that $Tx\leq Ty$.
\end{itemize}
Such maps are called \emph{order-preserving}.
\begin{definition}
  A map $f : X \to X$ is called \emph{subhomogeneous} if $\lambda f(x) \leq f(\lambda x)$ for all $x \in X$ and $0 < \lambda  < 1$.
\end{definition}
\begin{example}
  Let us consider the $1$-dimensional case, in this case $a,b,x \in \R$. Let us consider the sigmoid $\sigma(x) = \frac{1}{1+e^{-x}}$ and prove that $\sigma(ax + b)$ is subhomogeneous for any $b > -2$ if $a > 0$.

  First assume that $\lambda \in (0,1]$. Call $f_\lambda (x) = \lambda \sigma(ax+b)$ and $g_\lambda(x) = \sigma(a\lambda x+ b)$.
  We will prove the strict inequality $f_\lambda(x) < g_\lambda(x)$, assume now that for an $x_0$ there is a $\lambda_{x_0} \in (0,1]$ such that $f_{\lambda_{x_0}}(x_0) = g_{\lambda_{x_0}}(x_0)$ then let us differentiate w.r.t $\lambda$ and see
  \begin{align*}
    \frac{d}{d\lambda} f_{\lambda}(x_0) \bigg \lvert_{\lambda = \lambda_{x_0}} &= \sigma(ax_0+b) = g_{1}(x_0) \\
    \frac{d}{d\lambda} g_{\lambda}(x_0) \bigg \lvert_{\lambda = \lambda_{x_0}} &= ax_0 \sigma'(a \lambda_{x_0} x_0 + b) = ax_0 g_{\lambda_{x_0}}(x_0)(1-\sigma(a \lambda_{x_0} x_0+b)) \\
		&= ax_0f_{\lambda_{x_0}}(x_0)(1-\sigma(a\lambda_{x_0} x_0+b)) \\
    &=ax_0 \lambda_{x_0}(1-\sigma(a\lambda_{x_0} x_0+b)) g_1(x_0) < g_1(x_0)
  \end{align*}
  where in the above we have used that
  \begin{align} \label{eq:crit}
    x(1-\sigma(x+b)) < 1
  \end{align}
  which only holds for $b > -2$ and can be proven by straightforward computation.
  Note that this implies that for $\lambda < \lambda_{x_0}$ we have $f_{\lambda}(x_0) < g_\lambda(x_0)$. Now note that when $\lambda = 1$ we have $f(x) = g(x)$ for all $x$, by the above argument the strict inequality carries to all $\lambda \in (0,1)$ for all $x$.
\end{example}

Notice that this example generalizes to positive matrices $a$ and vectors $x$ as long as the activation function is applied component-wise (as usual) and the vector $b$ is greater than $-2$ in each component.

\begin{example}
  Let us again consider the case $\sigma(x) = \frac{1}{1+e^{-x}}$ and prove that given a positive matrix $W$ then for \emph{any} $b$ the mapping $Tx = \sigma(Wx+b)$ is order preserving. By the definition of order preserving we need to prove that given $x \leq y$ we have $Tx \leq Ty$. Even though $Wx+b$ might not be mapped into $X$ if $b$ is not a positive vector, we still have that $(Wy+b)-(Wx+b) \in X$. This together with the monotonicity of $\sigma$ gives that $T$ is order preserving on the positive cone.
\end{example}

Collecting the above, we may thus formulate:

\begin{proposition}
Given a matrix $W$ with positive entries, a vector $b$ with entries larger than $-2$, and $\sigma(x) = \frac{1}{1+e^{-x}}$ the sigmoid activation function. Then the associated layer map  $T(x)=\sigma(Wx+b)$ is non-expansive with respect to the Thompson metric on the standard positive cone.
\end{proposition}

\begin{example}
  Consider now the ReLU activation function $\sigma(x) = \max\{x,0\}$. We affirm that the mapping $Tx = \sigma(Wx+b)$ is subhomogeneous if $b \geq 0$ and $W$ is arbitrary. To see this: for any $0<\lambda<1$, we have
  \begin{align*}
    \lambda T(x) = \sigma (W \lambda x + \lambda B) \leq \sigma (W \lambda x + B) = T(\lambda x).
  \end{align*}
  And as already remarked, if all entries of $W$ is positive, then $T$ is order-preserving as well.
\end{example}

We thus have some interesting examples of order-preserving and subhomogeneous layer transformations and since these properties are preserved under composition we get a rich bank of examples. As was recalled above, the positive cone admits a metric $d$, the Thompson metric, which is semi-invariant under such maps, i.e.
\[
	d(Tx,Tx')\leq d(x,x')
\]
for all $x,x'\in X$.

In the following theorem we consider the mappings $T: X \rightarrow X$ to be random in a stationary way and quite general, but we will keep the above examples in mind. The point is here that in order for our dynamics to have a well defined limit we consider the ``layers'' added to the beginning of the network instead of at the end, i.e. we are interested in
\begin{align*}
  x_n = T_1 T_2 \ldots, T_n x_0.
\end{align*}
In conclusion, applying \cref{multergthm} we get:
\begin{theorem} \label{thm1}
	Let $X$ be the positive cone in $\R^N$ and let $T_i : X \to X$ be a stationary sequence of maps that is order preserving and subhomogeneous.
	Let $x_n = T_1 T_2 \ldots T_n x_0$, for a fixed $x_0 \in X$. Then
  \[
    \lim_{n\rightarrow\infty} \sup_i |x_n(i)|^{1/n}=e^{\lambda}
  \]
  and there is a (random) coordinate $1\leq i_0 \leq d$ such that
  \[
    \lim_{n\rightarrow\infty}|x_{n}(i_0)|^{1/n}=e^{\lambda}.
  \]
\end{theorem}

The latter statement excludes a certain spiraling inside the cone.

\subsection{Unitary case}

In the case of positive models that are order preserving and subhomogeneous, we get two things, first of all there is essentially exponential growth of the components of the trajectories, secondly we where able to determine the direction of such a trajectory. The subhomogeneity and order preserving transforms allowed us to find a metric which made the maps semi-invariant, if we on the other hand loosen the restrictions and consider general transforms we have the problem of finding good metrics. However if we instead restrict the matrices $W$ to be unitary (spectral norm $1$) then even if we consider fairly general norms we still get a non-expansive mapping. In this case we can also deduce the specific form of the metric functional which gives us a very concrete result.

To describe our situation let us begin by taking $X=\mathbb{R}^{N}$ together with a norm $\|\cdot\|_N$ and layer maps $T(x)=W^T\sigma(Wx+b)$ such that the corresponding operator norm $\left\Vert W\right\Vert_N \leq1$, $b$ general and $\sigma$ which satisfy $\left\|\sigma(x)-\sigma(y)\right\|_N\leq\left\Vert x-y\right\Vert_N $ (i.e. Lipschitz with constant $1$, or non-expansive) when applied to vectors (component-wise as always). The point of having the layer transformations in the form given above, is that it is a popular layer type that is used in for instance ResNets (\cite{HZRS16}) and provides us with a layer that can span the entire vector space. Let us now remark on the $1$-Lipschitz condition of the activation function in our norm. Begin by noting that most used activation functions are $1$-Lipschitz with respect to the standard absolute value, then if we assume that the norm is monotone, i.e. $\|x\|_N \leq \|y\|_N$ if $|x_i| \leq |y_i|$ for all $i$, the activation function becomes $1$-Lipschitz in the norm $\|\cdot \|_N$. For the theorem below we also assume the unit ball in the norm is a strictly convex set. We get by applying \cref{multergthm} (in the form of Corollary 1.5 in \cite{GK20} ):

\begin{theorem}
  Let $(X=\mathbb{R}^{N},\|\cdot\|_N)$ be a normed vector space which has the above monotonicity property and strictly convex unit ball. Consider a stationary sequence of
  layer maps $T_n$ of the form $T(x)=W^T\sigma(Wx+b)$, $\left\Vert W\right\Vert_N \leq1$, $b \in X$, and $\sigma$ is $1$–Lipschitz when applied componentwise in $(X=\mathbb{R}^{N},\|\cdot\|_N)$. Then as $n\rightarrow\infty$ it holds that a.s. there exists a vector $v$ such that
  \[
    \frac{1}{n}T_{1}T_{2}...T_{n}x_{0}\rightarrow v.
  \]
  The vector v is a priori random but independent of the initial data
  $x_{0}$. The norm of $v$ is deterministic.
\end{theorem}
The above theorem is a consequence of the same theorem that gave rise to \cref{thm1}, but in this case, since we have a norm we get that the metric functional reduces to a dot-product (\cref{prop:smooth:norm}) and from this fact we can read off the explicit convergence in the above theorem, see \cite{Ka} for more details.

%%%%%%%%%%%%%%%%%%%%%%%%%%%
\subsection{An example of the reverse order}

Let $\Omega =[-1,1]^{N}$ and use TanH which map the whole vector
space diffeomorphically onto the open cube. Let $X$ be the set of
distance functions on $\Omega$ bi-Lipschitz equivalent to the original
distance $d_0$ defined by a norm $\left\Vert \cdot\right\Vert $. The metric $D$ on $X$ is:
\[
    D(d_{1},d_{2})=\log\left(\max\left \{ \sup_{x\neq y}\frac{d_{2}(x,y)}{d_{1}(x,y)}, \sup_{x\neq y}\frac{d_{1}(x,y)}{d_{2}(x,y)} \right \}\right).
\]
which was discussed in \cref{sec:metric}, see also \cref{sub:Thompson-dist} for more details about the corresponding metric functionals.

We consider maps $T$ of the usual type except for insisting on that the matrices $W$ are invertible and also having fixed the activation function TanH so that
$T:\Omega \rightarrow \Omega$ and which will then be a diffeomorphism, and as remarked in \cref{sec:metric}, leaving the distance of the metric space $(X,D)$ invariant for the induced action $T^*$.

We consider the reverse dynamics $T_{n}T_{n-1}...T_{0}x_{0}$, and shift the
point of view by instead studying
\[
d_{n}:=T_{1}^{*}T_{2}^{*}...T_{n}^{*}d_0.
\]
In the above we see that the maps are ``inserted'' just before $d_0$, which is again the order which corresponds to random walks.

Thanks to the subadditive ergodic theorem and the other theorems in \cite{GK20}
$T_{1}^{*}T_{2}^{*}...T_{n}^{*}d_0$ will have some regular behavior
when $n\rightarrow\infty$ especially in terms of metric functionals on the space $X$. We have the following result proven in \cref{sub:Thompson-dist}.
\begin{theorem} \label{thm:exponential:expansion}
Under the above assumptions there is a number $\lambda$ so that
	\[
	\lim_{n\rightarrow\infty} \left(\sup_{x\neq y} \frac{\| T_nT_{n-1}...T_1x-T_nT_{n-1}...T_1y \|}
	{\|x-y \|} \right)^{1/n}=e^{\lambda}
	\]
	Moreover, in case $\lambda>0$ there exists a point $x \in \Omega$ and a sequence $z_i = (x_i,y_i) \in
	\{(x,y): x,y \in \Omega, x \neq y\}$ such that $z_i \to (x,x)$ and
	for any $\eps >0$ there is a number $N$ so that for $n>N$
    \begin{align*}
        \frac{\|T_n \ldots T_1 x_i - T_n \ldots T_1 y_i\|}{\|x_i-y_i\|} \geq e^{(\lambda-\eps)n},
    \end{align*}
    for all sufficiently large $i$.
\end{theorem}

The second assertion means in words that there is a random point $x\in \Omega $ and points approaching $x$ which the maps $T_nT_{n-1}...T_1$ separate with maximal exponential rate.

%%%%%%%%%%%%%%%%%%%%%%%%%%%
\subsection{Dynamics of decision functions}

Let $\Omega\subset \mathbb{R}^N$ be a compact set. We consider the dynamics $T_{n}T_{n-1}...T_{1}x$ for $x\in\Omega$ but shift the
point of view to study instead
\[
f_{n}:=T_{1}^{*}T_{2}^{*}...T_{n}^{*}f,
\]
where $f$ is the original decision function defined on $\Omega$. There are several possibilities here, especially with TanH and $\Omega=[-1,1]^N$, but we keep it simple and general to illustrate what can be shown. We assume that $f$ and all layer maps
are diffeomorphisms $\Omega\rightarrow \Omega$.

The maps are chosen in a stationary way as before. They preserve our Jacobi distortion metric $D$, from \cref{sec:metric}. It is then a standard fact that
$d(f,f_n)$ is a subadditive process. The subadditive ergodic theorem applies and gives since $J_f$ is bounded and bounded away from $0$:
\begin{theorem}
	In this situation there is a well-defined
	exponential growth rate $\lambda$ of the distortion of the decision functions $f_n(x)=f(T_nT_{n-1}...T_1x)$, more precisely,
	\[
	\lim_{n\rightarrow\infty} \sup_{x,y} |J_{f_n}(x)J_{f_n}(y)^{-1} |^{1/n}=e^{\lambda}.
	\]
\end{theorem}

Via the theory of random dynamical systems as exposed in \cite{Ar}, there
is an approach to the last theorem via Oseledets' theorem, see also \cite{Li}
for comparison.

\section{Short time (layer) behavior for random weight initialization}
\label{sec:cutoff}
When training deep neural networks an important aspect is how to do the random weight initialization such that we get a network that can actually be trained. In this section we will explore this concept a bit as it relates to our random layer transformations in this paper. However, in contrast to the previous theory we will study the ``short time'' behavior.

We will consider neural networks of the following simple type
\begin{align*}
    X_{i+1} = \sigma(W_i \cdot X_i), \quad i=0,\ldots
\end{align*}
where $W_i$ is i.i.d. from some distribution and $X_0$ is some starting point of the network. If we consider this as a Markov chain, i.e. for a fixed point we would take random weights at each step and consider the distribution of the output for a fixed $X_0$.

We will view this from the context of mixing and as such we need to frame our discussion with some terminology.

Consider a Markov chain on a domain of size $n$ (could be a graph or group for instance) and let $P_n^t(x, \cdot)$ be the distribution of the Markov chain started at $x$ at time $t$. Suppose that the Markov chain has a stationary distribution $\pi_n$. It is well known that if the Markov chain is irreducible and aperiodic we get exponential convergence, i.e. we get
\begin{align*}
    \max_{x} \|P_n^t(x,\cdot)-\pi_n\|_{TV} \leq e^{-ct}
\end{align*}
for some constant $c > 0$. Here $\| \cdot \|_{TV}$ is the total variation (TV) distance, which is defined for measures $\mu,\nu$ on $\Omega$ (finite state-space $\Omega$)
\begin{align*}
    \|\mu - \nu \|_{TV} = \sum_{x \in \Omega} |\mu(x) - \nu(y)|.
\end{align*}
If we define
\begin{align*}
    d_n(t) := \max_{x} \|P_n^t(x,\cdot)-\pi_n\|_{TV},
\end{align*}
then we say that the Markov chain exhibits a cut-off at $t_n$ with window $w_n$ if $w_n = o(t_n)$, and
\begin{align*}
    \lim_{\alpha \to \infty} \liminf_{n \to \infty} d_n(t_n - \alpha w_n) &= 1, \\
    \lim_{\alpha \to \infty} \limsup_{n \to \infty} d_n(t_n + \alpha w_n) &= 0.
\end{align*}
Intuitively speaking, this means that $d_n$ is close to $1$ for times just below $t_n$ and $d_n$ is close to 0 for times slightly larger than $t_n$, at least for large values of $n$. Note: the mixing time for the Markov chain will be close to $t_n$ for large $n$.
The mixing time for a Markov chain is defined as
\begin{align*}
    t_{mix}(\epsilon) &:= \min\{t: d(t) \leq \epsilon \} \\
    t_{mix} &:= t_{mix}(1/4).
\end{align*}
The cut-off phenomenon is stronger than the concept of fast mixing as it is a double sided property.

\subsection{Neural network induced Markov chains}

In the examples that we will simulate below, the limiting distribution is actually the point-mass at 0, and to make the total variation distance easier to define we work with finite precision, which makes the state-space finite. Note that if we have two measures $\mu$ and $\nu$ on a finite set $\Omega$ then,
\begin{align*}
    \|\mu - \nu \|_{TV} = \sum_{x \in \Omega} |\mu(x) - \nu(y)|
\end{align*}
i.e. the total variation distance is equivalent to the $L^1$ distance of the densities, which is easier to compute. We now come to our first example, namely a fully connected neural network with TanH activation and revisit the study done in \cite{Xavier}.

Consider the following simple Markov chain of neural network type (with heuristic initialization, see \cite{Xavier})
\begin{align} \label{eq:Markov}
    X_{i+1} = \tanh(W_i \cdot X_i), \quad i=0,\ldots
\end{align}
where $X_0 \neq 0, X_0 \in \R^N$, where $W_i$ are i.i.d. $W_i \sim \text{unif}\left (\left [-\frac{1}{\sqrt{N}},\frac{1}{\sqrt{N}} \right ]^N \right )$, and the TanH is applied componentwise (as is customary).

Above we can think of $N$ as the ``size'' of the Markov chain in the sense above. The result of the simulation can be found in \cref{fig:threshold}, where we worked with a finite precision of $0.001$ and measure the total variation distance to the point-mass at $0$. With the heuristic scaling factor introduced i.e. that $W_i \sim \text{unif}\left (\left [-\frac{1}{\sqrt{N}},\frac{1}{\sqrt{N}} \right ]^N \right )$ we see that they exhibit a cut-off at pretty much the same level. The cut-off implies that the behavior of this random initialization is markedly different for a layer count of around $3$ to layer counts above $10$.

It seems that this phenomenon occurs even for asymmetric activation functions, like the Sigmoid-weighted Linear Unit (SiLU or Swish) (\cite{Swish1,Swish2})
\begin{align} \label{eq:smooth:relu}
    \sigma(x) = \frac{x}{1+e^{-x}}
\end{align}
with the same setup as in \cref{eq:Markov} with the above activation, see \cref{fig:threshold2}. It does not however seem to occur for non-smooth activations, like ReLU.

\begin{figure}
    \centering
    \includegraphics[width=5cm]{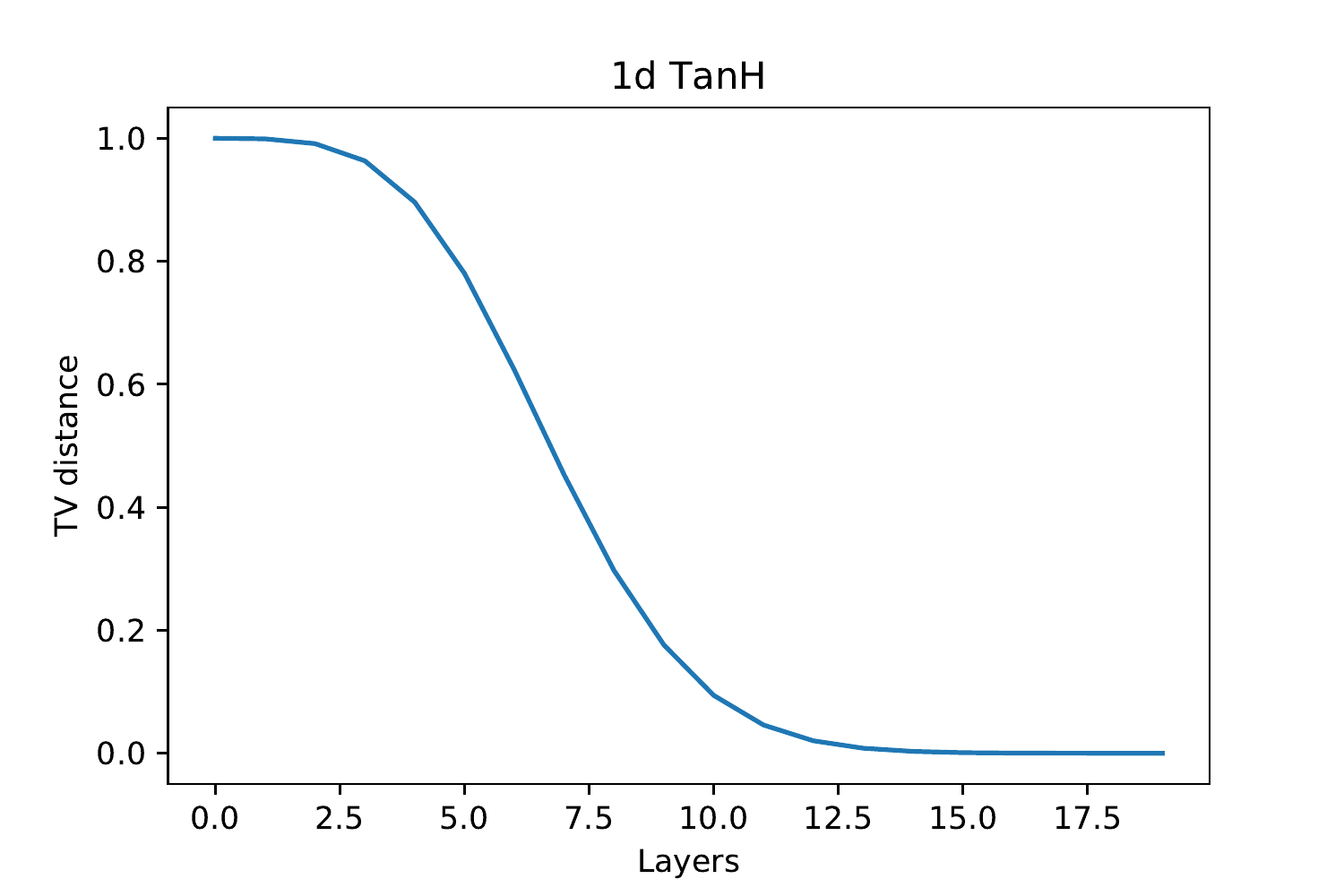}
    \includegraphics[width=5cm]{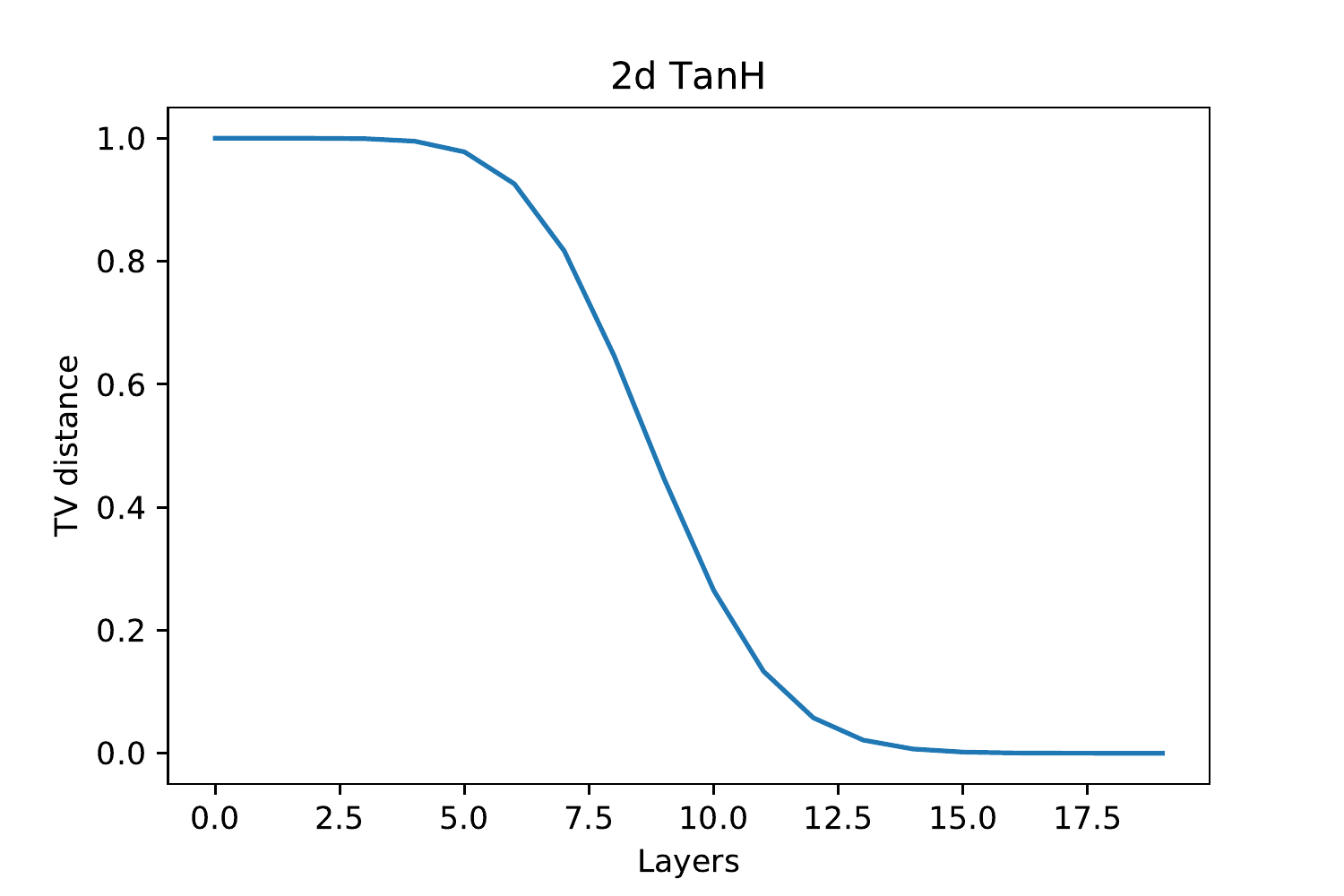}
    \caption{The cut-off phenomenon for Neural network mixing. The leftmost figure is width 1 and the rightmost figure is width 2. Mixing times are 9 and 11 respectively.}
    \label{fig:threshold}
\end{figure}

\begin{figure}
    \centering
    \includegraphics[width=5cm]{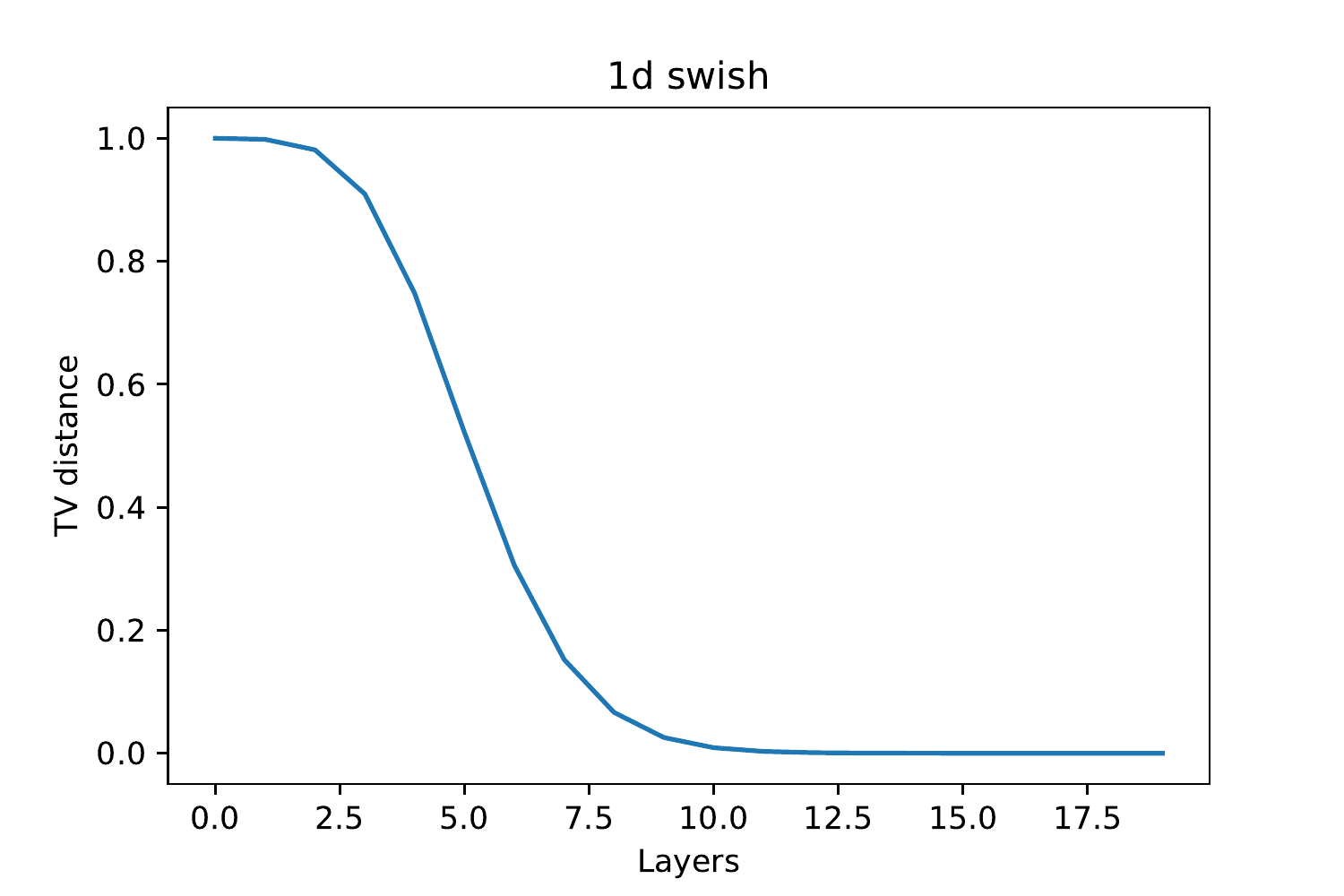}
    \caption{The cut-off phenomenon for Neural network mixing. The figure corresponds to activation function \cref{eq:smooth:relu}. Mixing time is 7.}
    \label{fig:threshold2}
\end{figure}

\section{Conclusion and outlook}

We have shown that some aspect of the understanding of deep learning
networks have a very natural dynamical interpretation, where recent
ergodic theorems can be applied. Indeed the tools are quite general
and therefore there is a wealth of possibilities. In other mathematical
contexts this versatility has already been demonstrated.

In the context of deep learning, one should translate the meaning
of the metrics and their functionals (say in terms of notions of complexity). This we have achieved for several choices of metrics.

A question that arises from the cut-off phenomenon that we demonstrate experimentally is its relevance for training of the network, that is, what difference it makes in practice from choosing fewer layers than the cut-off depth vis-à-vis choosing more layers. In fact, the fast convergence towards the point mass at $0$ hints that the last layers in a deep network will have an activation close to $0$, which is the linear regime of the activation function, implying that the deeper layers are basically linear mappings.
Furthermore, in \cite{Xavier} they considered instead the normalization which in our case becomes $\frac{\sqrt{3}}{\sqrt{N}}$ (as the input is the same as the output dimension) which actually only delays the cut-off to higher layer counts, it is however still there. One could speculate that the variance of the initialization can as such be used to control how nonlinear the initialized neural network is.

We believe that the metric setting will in forthcoming work have other interests
in deep learning, not only from the application of the non-commutative ergodic theorem. This could ultimately inform the choice of best design of the neural network for a given practical task.

\newpage
\appendix
\section{Metric functionals} \label{app:metric:functionals}

The mathematical abstraction of distance is that of a metric space $X$ which is a set with a distance function $d(x,y)$ that is symmetric, positive
and zero if and only if $x=y$. Moreover there is the fundamental triangle inequality: the distance between $x$ and $y$ cannot be larger than the sum of the distances from
$x$ to $z$ and from $z$ to $y$, for any point $z$. Sometimes it is useful and natural to relax the condition of symmetry.

We will now define the metric space analogs of linear functionals, affine hyperplanes and half spaces.
Let $(X,d)$ denote a metric space, fix $x_0 \in X$. Let $F(X,\mathbb{R})$ be the space of continuous functions $X\rightarrow \mathbb{R}$ equipped with the topology of
pointwise convergence. We define the continuous injection
	\begin{align*}
	  &\Phi : X \hookrightarrow F(X,\R) \\
	  &x \mapsto h_x(\cdot) := d(\cdot, x) - d(x_0,x)
	\end{align*}
The functions $h_x$ are all non-expansive with respect to $d$ and vanish at $x_0$. The image $\Phi(X)$ can be identified with a subset of a product of compact intervals, which is compact by Tychonoff's theorem. The closure of the image $\overline{\Phi(X)}$ will therefore be compact (similar to the compactness in the weak topology of functional analysis). See for example
\cite{GV} or \cite{Ka} for details. We will call the elements in this compact space \emph{metric functionals}. In particular, to each point $x$ there is the corresponding metric functional $h_x$. For metric functionals that are genuine limits, their level-sets are called horospheres, and sublevel sets are called horoballs. These two concepts are the metric analogs of affine hyperplanes and half-spaces in the linear vector space setting.

The first metric space to look at is the the finite dimensional euclidean space, where the metric is
given by $d(x,y)=||x-y||$ where the norm comes from a scalar product.
 In this case the metric functionals are up to a constant: the distance to a point in $X$ (in which case sublevel sets are balls) or linear functionals of norm 1 (in which case sublevel sets are half-spaces). Horospheres and horoballs are affine hyperplanes and half-spaces respectively.

\subsection{Metric functionals in the case of smooth norms}
Below we provide a characterization of the metric functionals in case of norms. For another proof for $p$-norms see \cite{Gutierrez}.
\begin{proposition}
    \label{prop:smooth:norm}
    Let $\| \cdot \|_a$ be a norm on $\R^N$ that is $C^2$ as a function on the corresponding unit sphere. Consider the function
    \begin{align*}
    	h_{y^n}(x) = \|y^n-x\|_a-\|y^n\|_a
    \end{align*}
    where $\|y^n\|_a \to \infty$ as $n\rightarrow \infty$. Then there is a subsequence and a vector $w$ with $\|w\|_a = 1$ such that
    \begin{align*}
	    \lim_{n \to \infty} h_{y^n}(x) = -x \cdot \left (\nabla \| \cdot \|_a \bigg \rvert_{w} \right ).
    \end{align*}
\end{proposition}
\begin{proof}
    Begin by first noting that $w^n = y^n/\|y^n\|_a$ is a vector on the unit sphere and as such there is a subsequence of $\{w^n\}$ converging to a vector $w$ s.t. $\|w\|_a = 1$. We will dispense with notation for subsequences for simplicity. Let us rewrite
    \begin{align*}
    	h_{y^n}(x) = \|y^n-x\|_a-\|y^n\|_a = \frac{\left \|w^n - \frac{x}{\|y^n\|_a}\right \|_a-\|w^n\|_a}{1/\|y^n\|_a}
    \end{align*}
    Taylor expanding the norm $\| \cdot \|_a$ around $w^n$
    \begin{align*}
    	\| z \|_a = \|w^n\|_a + (z-w^n) \cdot \nabla \| \cdot \|_a \rvert_{w} + f_n(z)\|z-w^n\|_2
    \end{align*}
    where $f_n(z) \to 0$ as $z \to w^n$, and $\|\cdot\|_2$ is the standard Euclidean norm. For simplicity let us relabel our sequence such that $\|y^n\|_a = n$, then
    \begin{align*}
    	\frac{\left \|w^n - \frac{x}{n}\right \|_a-\|w^n\|_a}{1/n} &= -x \cdot \left (\nabla \| \cdot \|_a \bigg \rvert_{w} \right ) + \frac{f_n(x/n)\|x/n\|_2}{1/n} \\
    	&=-x \cdot \left (\nabla \| \cdot \|_a \bigg \rvert_{w} \right ) + f_n(x/n)\|x\|_2 \\
    \end{align*}
    Together with a diagonal argument it is now clear that the $C^2$ regularity of $\| \cdot \|_a$ gives us the limit as $n \to \infty$ (for a subsequence)
    \begin{align*}
    	\lim_{n \to \infty} h_{y^n}(x) = -x \cdot \left (\nabla \| \cdot \|_a \bigg \rvert_{w} \right ).
    \end{align*}
\end{proof}

%%%%%%%%%%%%%%%%%%%%%%%%%%%%%
\subsection{Metric functionals for the Thompson metric}

In the following we make the identification of a vector $x \in \R_+^m$ with a positive function $x:I=\{1,\ldots,m\} \to \R_+$. For example, $\mathbbm{1}(i)=1$ for all $i$.
Consider the following
\begin{align*}
    d_1(x,y) = \log \sup_I \left ( \frac{x}{y} \right )
\end{align*}
we wish to derive the limit of
\begin{align*}
    d_1(x,y_n) - d_1(\mathbbm{1},y_n)
\end{align*}
where $d(\mathbbm{1},y_n) \to \infty$ as $n \to \infty$.

Consider the following
\begin{align*}
    d_1(x,y_n) - d_1(\mathbbm{1},y_n) &= \sup_I \left ( \log(x)+\log(\mathbbm{1}/y_n) - \sup_I \log(\mathbbm{1}/y_n) \right ) \\
    &=\sup_I \left ( \log(x)+\log\left (\frac{\mathbbm{1}/y_n}{e^{\sup_I \log(\mathbbm{1}/y_n)}} \right ) \right )
\end{align*}
the function
\begin{align*}
    u_n = \frac{\mathbbm{1}/y_n}{e^{\sup_I \log(\mathbbm{1}/y_n)}}
\end{align*}
satisfies $ 0 < u_n \leq 1$ and $\sup_I u_n = 1$. Therefore there is a subsequence of $u_n$ that converges (since $I$ is finite) to $u$ and we have that the same subsequence satisfies
\begin{align*}
    d_1(x,y_n) - d_1(\mathbbm{1},y_n) \to \sup_I \left ( \log(x)+\log (u) \right )
\end{align*}

The other part of the Thompson metric satisfies a similar relation, i.e.
let
\begin{align*}
    d_2(x,y) = \log \sup_I \left ( \frac{y}{x} \right )
\end{align*}
we wish to derive the limit of
\begin{align*}
    d_2(x,y_n) - d_2(\mathbbm{1},y_n)
\end{align*}
where $d(\mathbbm{1},y_n) \to \infty$ as $n \to \infty$.

Consider the following
\begin{align*}
    d_2(x,y_n) - d_2(\mathbbm{1},y_n) =\sup_I \left ( \log(1/x)+\log\left (\frac{y_n}{e^{\sup_I \log(y_n)}} \right ) \right )
\end{align*}
the function
\begin{align*}
    v_n = \frac{y_n}{e^{\sup_I \log(y_n)}}
\end{align*}
satisfies $ 0 < v_n \leq 1$ and $\sup_I v_n = 1$. Therefore there is a subsequence of $v_n$ that converges to $v$ (since $I$ is finite) and we have that the same subsequence satisfies
\begin{align*}
    d_2(x,y_n) - d_2(\mathbbm{1},y_n) \to \sup_I \left ( \log(1/x)+\log (v) \right )
\end{align*}

This discussion can be compared with a similar one in \cite{Gutierrez}. We summarize everything in the following proposition:
\begin{proposition}
    The metric functionals of the either half of the (or the full metric) Thompson metric that arise as limits (also called horofunctions) are given as follows:
    For $d_1$ we get that there exists a non-zero function $u:I \to [0,1]$ such that
    \begin{align*}
        h_{u}(x) = \sup_{I} \log(x u)
    \end{align*}
    and $\sup_I u = 1$.
    For $d_2$ we get the existence of a non-zero function $v: I \to [0,1]$ such that
    \begin{align*}
        h_{v}(x) = \sup_{I} \log(v/x)
    \end{align*}
    and $\sup_I v = 1$.
    In the case of $d = \max\{d_1,d_2\}$ then we have
    \begin{align*}
        h_{u,v}(x)=\max \{\sup_I \left ( \log(x)+\log (u) \right ),  \sup_I \left ( \log(1/x)+\log (v) \right \},
    \end{align*}
    where $uv = 0$ and $\sup_I \max\{u,v\} = 1$.
\end{proposition}

\subsubsection{Extension of the ideas to infinite dimensional spaces}

Identifying the metric functionals in cases of infinite dimensional spaces is more subtle, see for example \cite{Gutierrez2}. Especially for the $L^\infty$ space. We will take a look at a special case when we can actually determine the boundary for the Thompson version of the distortion metric.

\begin{align*}
    d(f,g)=\max\left \{\sup_{x\in I} \log\left|\frac{g'(x)}{f'(x)} \right |, \sup_{x\in I} \log\left|\frac{f'(x)}{g'(x)} \right |\right \}
\end{align*}

We choose again as basepoint in our function space the identity function $\mathbbm{I}(x) = x$. Note that
\begin{align*}
    d(\mathbbm{I},g)=\|\log(|g'|)\|_\infty.
\end{align*}

\begin{lemma}
    Consider the Thompson version of the distortion metrics used above,
    \begin{align*}
        d(f,g)=\max\left \{\sup_{x\in I} \log\left|\frac{g'(x)}{f'(x)} \right |, \sup_{x\in I} \log\left|\frac{f'(x)}{g'(x)} \right |\right \}
    \end{align*}
    then if $g_n$ is s sequence of $C^2$ functions such that $d(\mathbbm{I},g_n) \to \infty$, satisfying the following differential relation
    \begin{align*}
        \sup_I |g_n''| \leq C_0 \sup_I |g_n'|,
    \end{align*}
    for some $C_0 > 0$ independent of $n$, then there exists functions $u,v$ such that there is a subsequence that converges as
    \begin{align*}
        (d(f,g_n) - d(\mathbbm{I},g_n)) \to \sup_I \max \left \{\log(|f'|u),\log \left (\frac{1}{|f'|} v \right ) \right \}.
    \end{align*}
    Furthermore the functions $u,v$ are Lipschitz with constant $C_0$ such that $\sup_I \max\{u,v\} = 1$.
\end{lemma}
\begin{proof}
    First note that
    \begin{align*}
        d(f,g_n) - d(\mathbbm{I},g_n) = \max\left \{\sup_{x\in I} \log\left|\frac{g_n'(x)}{e^{d(\mathbbm{I},g_n)}}\frac{1}{f'(x)} \right |, \sup_{x\in I} \log\left|f'(x)\frac{1/g_n'(x)}{e^{d(\mathbbm{I},g_n)}} \right |\right \}.
    \end{align*}
    Consider
    \begin{align*}
        u_n &= \frac{g_n'}{\exp(d(\mathbbm{I},g_n))}, \\
        v_n &= \frac{(1/g_n')}{\exp(d(\mathbbm{I},g_n))},
    \end{align*}
    then
    \begin{align*}
        |u_n|,|v_n| \leq \frac{e^{\pm \log(|g_n'|)}}{\exp(d(\mathbbm{I},g_n))} \leq 1.
    \end{align*}
    This implies that
    \begin{align*}
        |u_n'| &= \frac{|g_n''|}{\exp(d(\mathbbm{I},g_n))} \leq \frac{C_0\sup_I |g_n'|}{\exp(d(\mathbbm{I},g_n))} = C_0 \sup_I |u_n| \leq C_0 \\
        |v_n'| &= \frac{\left |\frac{g_n''}{(g_n')^2} \right |}{\exp(d(\mathbbm{I},g_n))} \leq \frac{C_0 \frac{1}{\sup_I |g_n'|}}{\exp(d(\mathbbm{I},g_n))} = C_0 \sup_I |v_n| \leq C_0.
    \end{align*}
    From this we get that there is a subsequence of $u_n,v_n$ that converges uniformly to $u_\infty,v_\infty$ which are Lipschitz with constant $C_0$. Furthermore, $\sup_I \max\{u_\infty,v_\infty\} = 1$.
\end{proof}
\begin{remark}
    Actually if we see the above proof, then the only thing we need is to make sure that $u_n',v_n'$ is uniformly continuous. This follows for instance if the modulus of continuity $\omega$ of $g'$ is bounded by $|g'|$.
\end{remark}

\subsection{A Thompson metric for distance functions: Proof of \cref{thm:exponential:expansion}}\label{sub:Thompson-dist}
Let $\Omega$ be a compact subset of a finite dimensional vector space with a norm $\|\cdot \|$.
Let us consider the space $X$ of metrics on $\Omega$ which are bi-Lipschitz equivalent to $d_0(x,y):=\|x-y \|$.
Specifically this means that $d \in X$ iff there exists a constant $C > 1$ such that
\begin{align*}
    \frac{1}{C} \|x-y\| \leq d(x,y) \leq C \|x-y\|, \quad \forall x,y \in \Omega.
\end{align*}
On the space $X$ we can consider a Thompson type metric, defined as
\begin{align*}
    D(d_1,d_2) = \log \left ( \max \left \{\sup_{x \neq y} \frac{d_1(x,y)}{d_2(x,y)},\sup_{x \neq y} \frac{d_2(x,y)}{d_1(x,y)} \right \} \right ), \quad d_1,d_2 \in X.
\end{align*}
To understand this metric, note that
\begin{align*}
    e^{-D(d_1,d_2)} d_1(x,y) \leq d_2(x,y) \leq e^{D(d_1,d_2)} d_1(x,y).
\end{align*}
It is now easy to see that $X$ is complete under the metric $D$.
Consider the mapping $\mathcal{F}: X \to L^\infty(\Omega \times \Omega) \cap C(\Omega \times \Omega \setminus \{x = y\})$, defined as
\begin{align*}
    \mathcal{F}(d)(x,y) = \log(d(x,y)) - \log(\|x-y\|),
\end{align*}
denote $Y = \mathcal{F}(X)$.
From the bi-Lipschitz condition we see that
$ \sup_{x,y} | \mathcal{F}(d) (x,y) |\leq \infty$,
furthermore
\begin{align*}
    \|\mathcal{F}d_1 - \mathcal{F}d_2\|_{L^\infty(\Omega \times \Omega)} = D(d_1,d_2).
\end{align*}
As such, we see that the mapping $\mathcal{F}$ is an isometric mapping of $X$ into $L^\infty(\Omega \times \Omega)$ with respect to the canonical metric on the Banach space $L^\infty(\Omega \times \Omega)$.

Note also, again since the $L^{\infty}$-norm does not change on a null set, that we may write
\begin{align*}
    \|f\|=\max \{ p(f),p(-f)\},
\end{align*}
where $p(f)$ is the hemi-norm (cf. \cite{GV})
\begin{align*}
    p(f)=\esssup_{z\in \Omega\times \Omega}e_z(f),
\end{align*}
where $e_z$ is the evaluation functional $e_z(f)=f(z)$. In the case where $f$ is continuous on $\Omega \times \Omega \setminus \{x = y\}$ we can write
\begin{align*}
    p(f)=\sup_{z\in \Omega\times \Omega \setminus \{x = y\} }e_z(f).
\end{align*}

Now we introduce layer maps. More precisely, we consider
maps $T:\Omega \rightarrow \Omega$ which are injective. As
explained above these induce non-expansive maps in the metric $D$:
\begin{align*}
    (T^*d)(x,y):=d(Tx,Ty).
\end{align*}
To any mapping $U:X \to X$ there is a corresponding map $U \mathcal{F}(d) = \mathcal{F}(Ud)$, and, due to the isometry property of $\mathcal{F}$, a non-expansive mapping on $X$ becomes a non-expansive mapping on $Y$. We take as usual a stationary sequence of such layer maps
$T_n$ and denote by $f_n=\mathcal F (T_n...T_1)^*(d_0)$, where $d_0$ denotes the metric coming from the initial norm. Note that $f_0 = \mathcal F d_0 = 0$.

We note that $a(0,n)=p(f_n-f_0) = p(f_n)$ is a subadditive process, or subadditive cocycle in the terminology of \cite{GK20} (hemi-metrics work the same since only the triangle inequality and the non-expansiveness is used). By the subadditive ergodic theorem there is a number $\lambda$ such that
\begin{align*}
    \lim_{n\rightarrow\infty}\frac 1n p(f_n)=\lambda.
\end{align*}
Moreover,  \cite[Theorem 1.1]{GK20} asserts that for any decreasing positive sequence $\delta_n\rightarrow 0$ there are times $n_i\rightarrow \infty$ such that for every $i$ and $n<n_i$
\begin{align*}
    p(f_{n_i}-f_{n})-p(f_{n_i}) \leq -n(\lambda-\delta_n)
\end{align*}
By compactness we may moreover assume that $h_{n_i}(g) = p(f_{n_i}  - g) - p(f_{n_i})$ converges to a metric functional $h$.

We now follow a similar reasoning to \cite[p. 349]{GV}. From the continuity off the diagonal for elements in $Y$, we see that, given $f,g\in Y$ and $\delta>0$ there is an off-diagonal point $\hat z=(x,y)$ independent of $g$ such that
\begin{align} \label{eq:def:z}
    p(f-g)-p(f)\geq p(f-g)-e_{\hat z}(f)-\delta \geq e_{\hat z}(f-g)-e_{\hat z}(f) -\delta=e_{\hat z}(-g)-\delta.
\end{align}
Let $z_i$ be the off-diagonal points from \cref{eq:def:z} corresponding to $h_{n_i}$ in the inequality above now with $\delta_{n_i}$.
By again passing to a subsequence we can ensure that the corresponding sequence of points $z_i$ converges to a point $z$ thanks to compactness of $\Omega\times \Omega$ and $\delta_i\rightarrow 0$.

We now end up with two cases, either the limit point $z$ is on the diagonal or it is off diagonal. Let us begin with the off-diagonal case:
If $z=(x,y)$ is off-diagonal, then \cref{eq:def:z} gives that
\begin{align*}
    h(g)\geq - e_z (g).
\end{align*}
This implies in view of the multiplicative ergodic theorem in \cite{GK20} that
\begin{align*}
    \lim_{n\rightarrow \infty} \frac 1n f_n(z)=\lambda,
\end{align*}
and more concretely,
\begin{align*}
    \lim_{n\rightarrow \infty} d_n(x,y)^{1/n}=e^{\lambda}.
\end{align*}
Since $d_n(x,y)=(T_n...T_1)^*(d_0)(x,y)=\|T_n...T_1x-T_n...T_1 y\|$ and $\Omega$ is bounded in $d_0$, we can only have this conclusion in case
$\lambda=0$.

In the second case, when $z=(x,x)$ is on the diagonal, then in the notation $z_i =(x_i,y_i)$ we get from above on the one hand, for fixed $f_n$ and
all $n_i>n$
\begin{align*}
    p(f_{n_i}-f_n)-p(f_{n_i})\geq -e_{z_i}(f_n)-\delta_{n}
\end{align*}
and on the other hand
\begin{align*}
    p(f_{n_i}-f_n)-p(f_{n_i})\leq -n(\lambda-\delta_n).
\end{align*}
This implies that
\begin{align*}
    -n(\lambda-\delta_n)\geq -\log d_n(x_i,y_i)+\log \| x_i-y_i \| -\delta_n
\end{align*}
\begin{align*}
    d_n(x_i,y_i)\geq \| x_i-y_i \| e^{n(\lambda-\delta_n(1+1/n))}
\end{align*}
Choose $\epsilon$, then choose $N$ s.t. $\delta_N(1+1/N) < \epsilon$, then for $n_i > n > N$ we get
\begin{align*}
    d_n(x_i,y_i)\geq \| x_i-y_i \| e^{n(\lambda-\delta_n(1+1/n))} \geq \| x_i-y_i \| e^{n(\lambda-\epsilon)}.
\end{align*}
In words this means that there are sequences of points $x_i\rightarrow x$ and $y_i\rightarrow x$
which realize the growth rate of the Lipschitz constant, or put even more strikingly, there is a point $x\in\Omega$
such that nearby points are separated by the maximum amount ($\lambda$) by the maps $T_nT_{n-1}...T_1$.

\section*{Acknowledgments}
The first author was supported by the Swedish Research Council grant dnr: 2019-04098. The second author was partly supported by Swiss NSF grant 200020\_159581.

\newpage
\bibliography{references}

\end{document}